\documentclass[letterpaper]{article} 
\usepackage{aaai23}  
\usepackage{times}  
\usepackage{helvet}  
\usepackage{courier}  
\usepackage[hyphens]{url}  
\usepackage{graphicx} 
\usepackage{natbib}  
\usepackage{caption} 
\DeclareCaptionStyle{ruled}{labelfont=normalfont,labelsep=colon,strut=off} 
\usepackage{booktabs}
\usepackage{algorithm}
\usepackage{algorithmic}
\usepackage{multirow}
\usepackage{subcaption}

\usepackage[english]{babel}
\usepackage{amssymb}
\usepackage{amsfonts}
\usepackage{mathrsfs}
\usepackage{amsmath}
\usepackage{mathbbol}
\usepackage{booktabs}
\usepackage{bbm}
\usepackage{algorithm}
\usepackage{algorithmic}
\usepackage{makecell}
\usepackage{multirow}
\usepackage{amsmath}
\usepackage[switch]{lineno}
\usepackage{color}
\usepackage{newfloat}
\usepackage{listings}
\usepackage{epsfig,txfonts,threeparttable}

\newtheorem{definition}{Definition}

\newtheorem{lemma}{Lemma}
\newtheorem{theorem}{Theorem}

\newtheorem{proof}{Proof}

\usepackage{newfloat}
\usepackage{listings}
\floatstyle{ruled}
\newfloat{listing}{tb}{lst}{}
\floatname{listing}{Listing}
\newcommand{\argmin}{\mathop{{\rm arg}\min}}

\newcommand{\mbf}[1]{\mathbf{#1}}
\newcommand{\mbb}[1]{\mathbb{#1}}

\makeatletter
\newcommand{\ostar}{\mathbin{\mathpalette\make@circled*}}
\newcommand{\make@circled}[2]{%
	\ooalign{$\m@th#1\smallbigcirc{#1}$\cr\hidewidth$\m@th#1#2$\hidewidth\cr}%
}
\newcommand{\smallbigcirc}[1]{%
	\vcenter{\hbox{\scalebox{0.77778}{$\m@th#1\bigcirc$}}}%
}
\makeatother
\usepackage{color}
\title{Global Weighted Tensor Nuclear Norm for Tensor Robust Principal Component Analysis}

\author {
   Libin Wang, Yulong Wang, Shiyuan Wang, Youheng Liu, Yutao Hu, Hong Chen
}
\affiliations {
}

\begin{document}

\maketitle
\begin{abstract}
Tensor Robust Principal Component Analysis (TRPCA), which aims to recover a low-rank tensor corrupted by sparse noise, has attracted much attention in many real applications. This paper develops a new Global Weighted TRPCA method (GWTRPCA), which is the first approach simultaneously considers the significance of intra-frontal slice and inter-frontal slice singular values in the Fourier domain. Exploiting this global information, GWTRPCA penalizes the larger singular values less and assigns smaller weights to them. Hence, our method can recover the low-tubal-rank components more exactly. Moreover, we propose an effective adaptive weight learning strategy by a Modified Cauchy Estimator (MCE) since the weight setting plays a crucial role in the success of GWTRPCA. To implement the GWTRPCA method, we devise an optimization algorithm using an Alternating Direction Method of Multipliers (ADMM) method. Experiments on real-world datasets validate the effectiveness of our proposed method.
\end{abstract}

\section{Introduction}
Real-world data such as images, texts, videos and bioinformatics are usually high-dimensional and can be approximated by low-rank structures. Exploiting these low-rank structures from high-dimensional data is an essential problem in many real applications, e.g., face recogniton \cite{face}, collaborative filtering \cite{Koren08} and image denoising \cite{Bouwmans18}. Robust Principal Component Analysis (RPCA) \cite{RPCA}, one of the most representative methods in this problem, has attracted much attention. It aims to separate the observed data matrix $\mbf{X}\in\mbb{R}^{d_1\times d_2}$ into a low-rank matrix $\mbf{L}$ and a sparse matrix $\mbf{E}$, where $\mbf{L}$ denotes clean data and $\mbf{E}$ denotes the noise. It has been proved in \cite{RPCA} that $\mbf{L}$ and $\mbf{E}$ can be exactly recovered with high probability under proper conditions by solving the following convex problem
\begin{equation}\label{Eq:RPCA}
	\min_{\mbf{L}, \mbf{E}\in\mbb{R}^{d_1\times d_2}}\|\mbf{L}\|_*+\lambda\|\mbf{E}\|_1
	~~\text{s.t.}~~\mbf{X}=\mbf{L}+\mbf{E},
\end{equation}
where $ \lambda >0$ is the regularization parameter,  $\|\mbf{\cdot}\|_*$ and $\|\mbf{\cdot}\|_1$ denote the nuclear norm (sum of the singular values) and $\ell_1$-norm (sum of the absolute values of all entries), respectively.

\begin{figure}	\includegraphics[width=8cm,height=2cm]{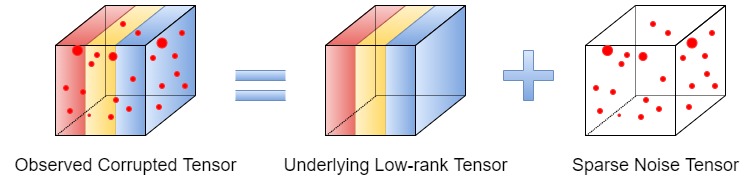}
	\caption{An illustration of Tensor Robust Principal Component Analysis.}
	\label{Fig:TRPCA}
\end{figure}  
One major drawback of RPCA is that it can only handle 2-order (matrix) data. However, multi-dimensional data, termed as tensor data \cite{Kolda09}, is widely used in many real applications \cite{Gao20AAAI,Mu20PR}. For example, a color image is a 3-order tensor with column, row and color modes; a gray scale video is indexed by two spatial variables and one temporal variable. To handle the tensor data, RPCA needs to restructure the high-order tensor data into a matrix and thus ignores the information embedded in multi-dimensional structures. The information loss can lead to performance degradation and an inexact recovery. 

To alleviate this problem, many Tensor RPCA methods have been proposed, which aim to separate the observed tensor data $\mathcal{X}$ into a clean low-rank tensor $\mathcal{L}$ and a sparse noise tensor $\mathcal{E}$ (An illustration can be seen in Fig. \ref{Fig:TRPCA}). \citet{Liu13TPAMI} initially extended the matrix nuclear norm to tensors by proposing Sum of Nuclear Norms (SNN). It is defined as $\sum_i{\|\mbf{X}^{\{i\}}\|}_*$, where $ \mbf{X}^{\{i\}} $ is the mode-$i $ matricization of $\mathcal{X}$ \cite{Liu13TPAMI}. The SNN based Tensor RPCA method \cite{SNN} can be formulated as
\begin{equation}\label{eq:SNN}
	\min_{\mathcal{L}, \mathcal{E} \in \mathbb{R}^{d_{1} \times d_{2} \times d_{3}}}\sum_{i=1}^{k}\lambda_i{\|{\mbf{L}^{\{ i \}}}\|}_{*}+{\|\mathcal{E}\|}_{1}~~\text{s.t.}~~\mathcal{X}=\mathcal{L}+\mathcal{E},
\end{equation}
where $ \lambda_i >0$ are regularization parameters. Although the SNN based Tensor RPCA method can guarantee the recovery, the algorithm is nonconvex. To overcome this shortcoming, \citet{TRPCA} proposed the Tensor Robust Principal Component (TRPCA) method with a new tensor nuclear norm, which is motivated by the tensor Singular Value Decomposition (t-SVD) \cite{Kilmer11} :
\begin{equation}
	\min _{\mathcal{L}, \mathcal{E} \in \mathbb{R}^{d_{1} \times d_{2} \times d_{3}}}\|\mathcal{L}\|_{*}+\lambda\|\mathcal{E}\|_{1}~~\text {s.t.}~~ \mathcal{X}=\mathcal{L}+\mathcal{E},
\end{equation}
where $\|\cdot\|_*$ denotes the tensor nuclear norm (TNN).

Inspired by TNN defined in TRPCA, many improvements of TRPCA have been proposed. \citet{ETRPCA} presented an Enhanced TRPCA (ETRPCA) method, which defines a weighted tensor Schatten p-norm to deal with the difference between singular values of tensor data. \citet{FTNN} developed a Frequency-Filtered TRPCA mehtod. This method defines a frequency-filtered TNN, which utilizes the prior knowledge between different frequency bands. Besides, \citet{Liu2018} indicated that TRPCA fails to employ the low-rank structure in the third mode and extending TNN with core matrix. \citet{Jiang20JCAM} defined a partial sum of tubal nuclear norm (PSTNN) of a tensor, which is a surrogate of the tensor tubal multi-rank. By minimizing PSTNN, the smaller singular values are shrinked while the larger singular values are not. Additionally, TNN may over-penalize the large singular values of tensor data \cite{Li21ICPR}.
\citet{Kong2018} proposed t-Schatten-$p$ quasi-norm to improve TNN, which is non-convex when $ 0<p<1 $ and can be a better approximationof  the $l_1$ norm of tensor multi-rank. Besides, a new t-Gamma tensor quasi-norm as a non-convex regularzation was defined by \citet{Cai2019} to approximate the low-rank component.
\begin{figure}	\includegraphics[width=8cm,height=3cm]{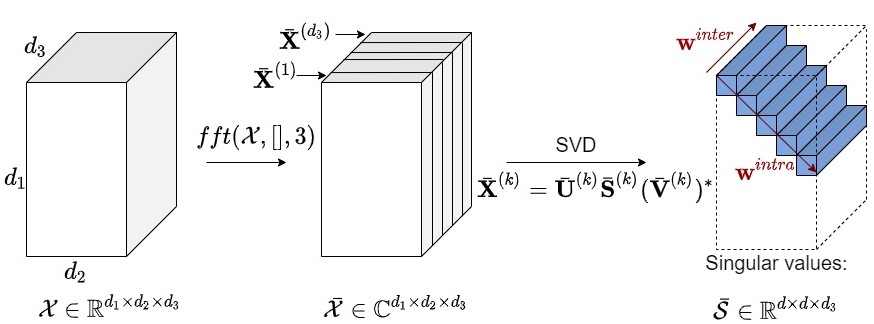}
	\caption{An illustration of the global weight to singular values of an $ d_1 \times d_2 \times d_3$ tensor in the Fourier domain ($d=\min(d_1,d_2)$; $ \mbf{w}^{intra} \in \mbb{R}^{d}$: weight vector for intra-frontal slice singular values; $ \mbf{w}^{inter} \in \mbb{R}^{d_3}$: weight vector for inter-frontal slice singular values).}
	\label{Fig:GWTNN}
\end{figure}
\subsection{Motivation and Contribution }
Despite the theoretical guarantee and emperical success, TRPCA treats all singular values of tensor data equally and regularizes different singular values with the same weights. It may cause an inexact estimation of the tubal rank of tensor data \cite{ETRPCA}. Additionally, singular values usually have clear physical meanings in many real applications. For example, larger singular values are generally associated with some prominent information of the image. Hence, we should assign smaller weights for the larger singular values to efficiently preserve the prominent information. 

To overcome the drawback of TRPCA, Enhanced TRPCA (ETRPCA) \cite{ETRPCA} assigns different weights to penalize distinct singular values. However, it only considers intra-frontal slice singular values. We observe that the decreasing property also holds on for inter-frontal slice singualr values (along the third dimension). As is shown in Fig. \ref{Fig:GWTNN}, we only exploit the local information of tensor singular values and neglect the inter-frontal slice weights when we just assign intra-frontal slice weights like ETRPCA. This can result in an inaccurate evaluation of weights and performance decline. To handle this problem, our method GWTRPCA defines a new Global Weighted TNN and can learn weights of different singular values more exactly.

The contributions of this paper are as follows: 
\begin{itemize}
	\item [1.] We propose a new weighted TNN, Global Weighted Tensor Nuclear Norm (GWTNN), which considers the significance of intra-frontal slice and inter-frontal slice singular values simultaneously in the Fourier domain. GWTNN can be a better surrogate of the tubal rank function. 
	\item [2.] We put forward a new TRPCA method employing GWTNN termed Global Weighted Tensor Principle Component Analysis (GWTRPCA). We also devise an effective optimization algorithm for GWTRPCA based on the ADMM framework.
	\item [3.] We propose an adaptive weight learning strategy by a Modified Cauchy Estimator (MCE). We utilize the strategy to learn inter-frontal slice weights and further improve the performance of GWTRPCA.
\end{itemize}

\section{Notations and Preliminaries}
\subsection{Notations}
	In this paper, we represent scalars, vectors, matrices and tensors using normal letters, boldface lower-case, upper-case letters and script letters, respectively. For a complex matrix $ \mbf{X} \in \mathbb{C}^{d_{1} \times d_{2}} $, $ \mbf{X}^* $ denote its conjugate transpose. We denote $ \lceil t \rceil $ as the nearest integer more than or equal to t. For $ t\in \{1,2,\cdots,n\} $, we denote it as $ t \in [n] $. 
	For brevity, we summarzie the main natations in Tab. \ref{tab:notation}.

\subsection{T-Product}
To define the tensor-tensor product (or t-product), we first introduce some useful tensor operators.
\begin{definition}\label{def:fold}(\textbf{The fold and unfold operators}  \cite{Kilmer11})
	For  any 3-order tensor $\mathcal{X}\in \mbb{R}^{d_1\times d_2\times d_3}$, we define
	\begin{equation}\label{eq:unfold}
		\textnormal{unfold}(\mathcal{X})=
		\begin{bmatrix}
			\mbf{X}^{(1)}  \\
			\mbf{X}^{(2)}  \\
			\vdots         \\
			\mbf{X}^{(d_3)}\\
		\end{bmatrix},~
		\textnormal{fold}(\textnormal{unfold}(\mathcal{X}))=\mathcal{X}.
	\end{equation}
	
\end{definition}

\begin{definition}\label{def:bcirc}(\textbf{The bcirc operator}  \cite{Kilmer11})
	The block circulant matrix of any 3-order tensor  $\mathcal{X}\in \mathbb{R}^{d_1\times d_2\times d_3}$
		is defined as
        \begin{equation}\label{eq:bcirc}
			\textnormal{bcirc}(\mathcal{X})=
			\begin{bmatrix}
				\mbf{X}^{(1)}    &\mbf{X}^{(d_3)}     &\cdots   &\mbf{X}^{(2)}   \\
				\mbf{X}^{(2)}    &\mbf{X}^{(1)}       &\cdots   &\mbf{X}^{(3)}   \\
				\vdots           &\vdots              &\ddots   &\vdots          \\
				\mbf{X}^{(d_3)}  &\mbf{X}^{(2)}       &\cdots   &\mbf{X}^{(1)}    \\
			\end{bmatrix}
		\end{equation}
		where $\mbf{X}^{(k)}$ denotes the $k$-th frontal slice of $\mathcal{X}$.
	\end{definition}

\begin{definition}\label{def:T-product}(\textbf{T-product} \cite{Kilmer11})
		The t-product of any two 3-order tensor  $\mathcal{X}\in \mbb{R}^{d_1\times d_2\times d_3}$ and
		$\mathcal{Y}\in \mbb{R}^{d_2\times l\times d_3}$ is
		defined to be a tensor
		of size $d_1\times l\times d_3$
		\begin{equation}\label{eq:T-product}
			\mathcal{X}*\mathcal{Y}=\textnormal{fold}(\textnormal{bcirc}
			(\mathcal{X})\cdot\textnormal{unfold}(\mathcal{Y})).
		\end{equation}
	\end{definition}
	In fact, the t-product $\mathcal{Z}=\mathcal{X}*\mathcal{Y}$  can be computed efficiently with the aid of $fft$ \cite{TRPCA}. Specifically,
	let $\bar{\mathcal{X}}=fft(\mathcal{X},[],3)$ and
	$\bar{\mathcal{Y}}=fft(\mathcal{Y},[],3)$.
	We can first calculate the $k$-th frontal slice of $\bar{\mathcal{Z}}$ as $\bar{\mbf{Z}}^{(k)}
	=\bar{\mbf{X}}^{(k)}\bar{\mbf{Y}}^{(k)}$,
	where $\bar{\mbf{X}}^{(k)}$ and $\bar{\mbf{Y}}^{(k)}$
	are the $k$-th frontal slice of $\bar{\mathcal{X}}$
	and $\bar{\mathcal{Y}}$, respectively.
	Then we can obtain $\mathcal{Z}=ifft(\bar{\mathcal{Z}},[],3)$.
\begin{table}
	\begin{center}
		\normalsize
		\begin{tabular}{l|l}
			\hline\cline{1-2}
			\textbf{Notation}       &\textbf{Description}\\ \hline
			$\mathcal{X} \in \mathbb{R}^{d_{1} \times d_{2} \times d_{3}}$                             &tensor data  \\
			$\mathcal{L} \in \mathbb{R}^{d_{1} \times d_{2} \times d_{3}}$                             &low-rank tensor  \\
			$\mathcal{S} \in \mathbb{R}^{d_{1} \times d_{2} \times d_{3}}$                              &sparse tensor  \\
			$\mathcal{X}_{ijk}$ or $\mathcal{X}(i,j,k)$                             &$(i,j,k)$-th entry of $\mathcal{X}$\\
			$\mathcal{X}(i,:,:)$
			&$i$-th horizontal slice of $\mathcal{X}$\\
			$\mathcal{X}(:,j,:)$
			&$j$-th lateral slice of $\mathcal{X}$\\
			$\mathcal{X}(:,:,k)$ or $\mbf{X}^{(k)} $
			&$k$-th frontal slice of $\mathcal{X}$\\
			DFT                         &Discrete Fourier Transform \\
			FFT                         &Fast Fourier Transform \\
			$\bar{\mathcal{X}}\in\mbb{C}^{d_1\times d_2\times d_3}$                         &DFT of $\mathcal{X}$ along the $3$-rd dimension\\
			$ fft(\mathcal{X},[],3) $ & FFT of $\mathcal{X}$ along the $3$-rd dimension\\
			$\|\mathcal{X}\|_{*}$  &Tensor Nuclear Norm \\
			$\|\mathcal{X}\|_{1}$  &Tensor $\ell_1$ Norm  \\
			TNN                  &Tensor Nuclear Norm \\
			GWTNN                 &Global Weighted TNN    \\
			$ [n] $                &set $ \{1,2,\cdots,n\} $    \\
			\hline\cline{1-2}
		\end{tabular}
	\end{center}
	\caption{Key notations used in this paper.}
	\label{tab:notation}
\end{table}
\subsection{T-SVD and Tensor Nuclear Norm}
\label{subsec:TNN}

Before introducing the Tensor-SVD (T-SVD)
and Tensor Nuclear Norm (TNN), it is necessary to
introduce some useful definitions associated with tensors.
\begin{definition}\label{def:Conj}(\textbf{Conjugate transpose} \cite{Kilmer11})
	The conjugate transpose of a tensor $\mathcal{X}\in \mbb{C}^{d_1\times d_2\times d_3}$ is the tensor $\mathcal{X}^* \in \mbb{C}^{d_2\times d_1\times d_3}$ obtained by conjugate transposing each of the frontal slices and then reversing the order of the transposed frontal slices 2 through $ d_3 $.
\end{definition}

\begin{definition}\label{def:Identity}(\textbf{Identity  tensor} \cite{Kilmer11})
	The identity tensor $\mathcal{I}\in \mbb{R}^{d\times d\times d_3}$
	is the tensor with its first frontal slice being the
	$d\times d$ identity matrix, and other frontal slices being all zeros.
\end{definition}

\begin{definition}\label{def:Orth}(\textbf{Orthogonal tensor} \cite{Kilmer11})
	A tensor  $\mathcal{Q}\in \mbb{R}^{d\times d\times d_3}$ is orthogonal if it satisfies
	$\mathcal{Q}^**\mathcal{Q}=\mathcal{Q}*\mathcal{Q}^*=\mathcal{I}$.
\end{definition}

\begin{definition}\label{def:fdiag}(\textbf{F-diagonal tensor} \cite{Kilmer11})
	A tensor is called f-diagonal if each of its frontal slices is a diagonal matrix.
\end{definition}

\begin{definition}\label{def:T-SVD}(\textbf{Tensor Singular Value Decomposition, T-SVD} \cite{Kilmer11})
	The tensor $\mathcal{X}\in \mbb{R}^{d_1\times d_2\times d_3}$ can be factorized as
	\begin{equation}\label{eq:T-SVD}
		\mathcal{X}=\mathcal{U}*\mathcal{S}*\mathcal{V}^*,
	\end{equation}
	where $\mathcal{U}\in \mbb{R}^{d_1\times d_1\times d_3}$,
	$\mathcal{V}\in \mbb{R}^{d_2\times d_2\times d_3}$ are orthogonal,
	and $\mathcal{S}\in \mbb{R}^{d_1\times d_2\times d_3}$ is an f-diagonal tensor.
\end{definition}

\begin{definition}\label{def:TNN}(\textbf{Tensor Nuclear Norm} \cite{TRPCA})
	Let $\mathcal{X}=\mathcal{U}*\mathcal{S}*\mathcal{V}^*$ be the t-SVD of $\mathcal{X}\in \mbb{R}^{d_1\times d_2\times d_3}$
	and $d=\min(d_1,d_2)$.
	The  tensor nuclear norm of $\mathcal{X}$ is defined as
	\begin{equation}\label{eq:WNN}
		\|\mathcal{X}\|_{*}
		=\sum_{i=1}^d\mathcal{S}(i,i,1)
		=\frac{1}{d_3}\sum_{k=1}^{d_3}\sum_{i=1}^d
		\sigma_{i}\left(\bar{\mbf{X}}^{(k)}\right),
	\end{equation}
	where $\sigma_{i}\left(\bar{\mbf{X}}^{(k)}\right)$ is the $i$-th singular value of $\bar{\mbf{X}}^{(k)}$.
\end{definition}

\section{Proposed Method}
\subsection{Global Weighted Tensor Nuclear Norm}

It is known that the singular values of a matrix have the decreasing property. Given any tensor $ \mathcal{X} \in \mbb{R}^{d_1\times d_2\times d_3}$ and $d=\min\{d_1,d_2\} $, let $\bar{\mathcal{X}}$ be the result by applying DFT (Discrete Fourier Transform) on $\mathcal{X}$ along the third dimension,
$\bar{\mbf{X}}^{(k)}$ is the  $k$-th frontal slice of $\bar{\mathcal{X}}$,
and $\sigma_{i}\left(\bar{\mbf{X}}^{(k)}\right)$ is the $i$-th singular value of $\bar{\mbf{X}}^{(k)}$. Thus, for any $ k \in [d_3] $, the singular values in the  $k$-th frontal slice $\bar{\mbf{X}}^{(k)}$ have the same decreasing property, i.e.,
\begin{equation}\label{Decrease-intra}
	\sigma_{1}\left(\bar{\mbf{X}}^{(k)}\right) \geq \sigma_{2}\left(\bar{\mbf{X}}^{(k)}\right) \geq
	\cdots  \geq
	\sigma_{d}\left(\bar{\mbf{X}}^{(k)}\right) \geq 0
\end{equation}

\begin{figure}
	\includegraphics[width=8cm,height=2.8cm]{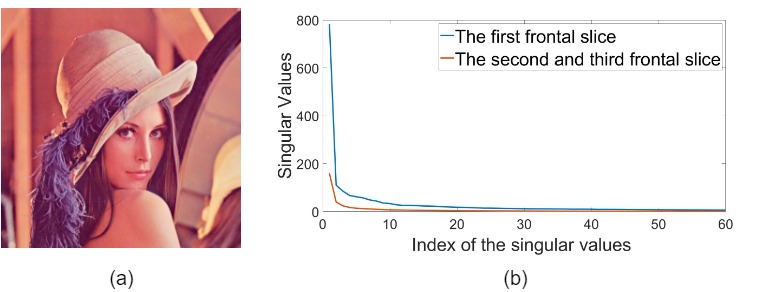}
	\caption{Illustration of the decreasing property of the inter-frontal slice singular values in the Fourier domain. (a) A color image can be modeled as a tensor $ \mathcal{M} \in \mbb{R}^{512\times 512\times 3}$; (b) plot of the first 60 singular values in the first, second and third frontal slice of $ \bar{\mathcal{M}} $, respectively. Note that the singular values in the 2nd frontal slice are equal to those in the 3rd due to the conjugate symmetry of FFT.
	}
	\label{fig:Lena}
\end{figure}

The decreasing property also holds on when we consider the inter-frontal slice (across different frontal slices) singular values. As is shown in Fig. \ref{fig:Lena}, the singular values in the first frontal slice tend to be larger than those in the second and third frontal slice for a 3-order tensor $ \bar{\mathcal{M}} $ in the Fourier domain. Otherwise, we observe that the singular values in the second frontal slice are equal to those in the third frontal slice. This can be due to the property of DFT. Hence, the singular values of distinct frontal slices $ \bar{\mathcal{X}}^{(k)} $ have another interesting property, i.e.,
\begin{equation}\label{Decrease-inter}
\sum_{i=1}^d \sigma_i\left(\bar{\mbf{X}}^{(1)}\right) \geq
\cdots \geq
\sum_{i=1}^d \sigma_i\left(\bar{\mbf{X}}^{(\left\lceil\frac{d_{3}+1}{2}\right\rceil)}\right) \geq 0 
\end{equation}

\begin{equation}\label{Decrease-symmetric}
\sum_{i=1}^d \sigma_i\left(\bar{\mbf{X}}^{(k)}\right)=\sum_{i=1}^d \sigma_i\left(\bar{\mbf{X}}^{(d_{3}-k+2)}\right), k= \left\lceil\frac{d_{3}+1}{2}\right\rceil+1, \cdots, d_{3} 
\end{equation}
Now we simultaneously consider the decreasing property of intra-frontal slice and inter-frontal slice singular values. Based on the global decreasing property, we can assign different weights to distinct singular values and give a better surrogate of the tubal rank with a new weighted tensor nuclear norm. Thus, we have the following definiton of Global Weighted Tensor Nuclear Norm (GWTNN).

It is worth noting that ETRPCA proposes the weighted tensor Schatten $ p$-norm (WTSN) $
\|\mathcal{X}\|_{\mbf{w},S_p}
=
\left(\sum_{k=1}^{d_3}\sum_{i=1}^d w_i
\sigma_{i}^p\left(\bar{\mbf{X}}^{(k)}\right)\right)^{\frac{1}{p}} 
$ for $p>0$. When we overlook the differences of inter-frontal slice singular values and equalize the weights $ w_{k}^{inter} = 1 $, GWTNN reduces to a special case of WTSN when $p=1$. Compared to WTSN, GWTNN possesses two advantages: First, GWTNN can be seen as a generalization of WTSN. GWTNN can approximate the tubal rank function more exactly and maintain the flexibility to solve different practical problems. Second, our proposed GWTNN considers the importance of different frontal slices in the Fourier domain. It can perserve the significant components in DFT, consequently improving the recovery performance.

\subsection{Optimization}
In this part, we develop an efficient optimization
algorithm based on the ADMM framework.

The Lagrangian function of the GWTRPCA method (\ref{eq:GWTRPCA}) is formulated as
\begin{equation}\label{eq:Lag}
	\begin{split}
		L({\mathcal{L}}, {\mathcal{E}}, \mathcal{Y}, \mu)
		=&\|{\mathcal{L}}\|_{\ostar}+\lambda \|{\mathcal{E}}\|_{1}\\
		&+\frac{\mu}{2}\left\|{\mathcal{L}}+{\mathcal{E}}-\mathcal{X}+\mathcal{Y}/\mu\right\|_F^2
		-\frac{\mu}{2}\|\mathcal{Y}/\mu\|_F^2.
	\end{split}
\end{equation}
The variables can be updated alternatively by fixing others in each iteration.
Let ${\mathcal{E}}_t$, $\mathcal{Y}_t$, $\mbf{w}^{iner}_{t}$ and $\mu_t$ be the value  of
${\mathcal{E}}$, $\mathcal{Y}$, $\mbf{w}^{iner}$ and $\mu$ in the $t$-th iteration, respectively.


\textbf{Step 1: } Update ${\mathcal{L}}$

\begin{equation}\label{eq:UpdateL}
	{\mathcal{L}}_{t+1}=\argmin_{{\mathcal{L}}\in\mbb{R}^{d_1\times d_2\times d_3}}
	\frac{1}{2}\left\|{\mathcal{L}}-\left(\mathcal{X}-{\mathcal{E}}_t-\mathcal{Y}_t/\mu_t\right)\right\|_F^2
	+\frac{1}{\mu_t}\|{\mathcal{L}}\|_{\ostar}.
\end{equation}

\begin{algorithm}[!t]
	\caption{GWTRPCA}
	\label{alg:algorithm1}
	\textbf{Input}: Data tensor $\mathcal{X}\in\mbb{R}^{d_1\times d_2\times d_3}$,
	weight vector $\mbf{w}^{intra}\in\mbb{R}^{d}, d=\min\{d_1,d_2\}$
	and the parameter $\lambda$. \\
	\textbf{Initialize}: ${\mathcal{L}}_0=0$, ${\mathcal{E}}_0=0$, $\mathcal{Y}_0=0$, $ \mbf{w}^{inter}=\mbf{1}$,
	$\mu=10^{-2}$, $\rho=1.1$, $\mu_{\text{max}}=10^7$,
	and $\epsilon=10^{-6}$.
	\begin{algorithmic}[1]    
		\WHILE{not convergence}
		\STATE Update ${\mathcal{L}}$ by Eq. (\ref{eq:UpdateL});\\
		\STATE Update ${\mathcal{E}}$ by Eq. (\ref{eq:UpdateS});
		\STATE Update the multipliers
		$
		\mathcal{Y}_{t+1}=\mathcal{Y}_t+\mu_t({\mathcal{L}}_{t+1}-{\mathcal{E}}_{t+1}-\mathcal{X})
		$;\\
		\STATE Update the parameter $\mu$ by
		$
		\mu_{t+1}=\min\{\rho\mu_t, \mu_{\text{max}}\};
		$
		\STATE Check the convergence conditions: \\
		$
		\|\mathcal{X}-{\mathcal{E}}_{t+1}-{\mathcal{L}}_{t+1}\|_{\infty}<\epsilon.
		$ \\
		\ENDWHILE
	\end{algorithmic}
	\textbf{Output}: ${\mathcal{L}}={\mathcal{L}}_{t+1}$, ${\mathcal{E}}={\mathcal{E}}_{t+1}$.
\end{algorithm}

\begin{figure*}[htbp]
	\centering
	\begin{subfigure}[b]{0.15\textwidth}
		\includegraphics[width=\linewidth]{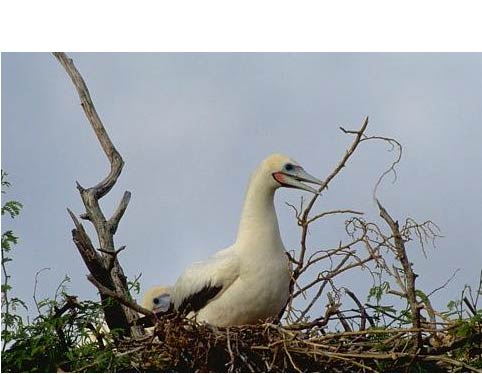}
		\caption{Oringinal}
	\end{subfigure}
	\begin{subfigure}[b]{0.15\textwidth}
		\includegraphics[width=\linewidth]{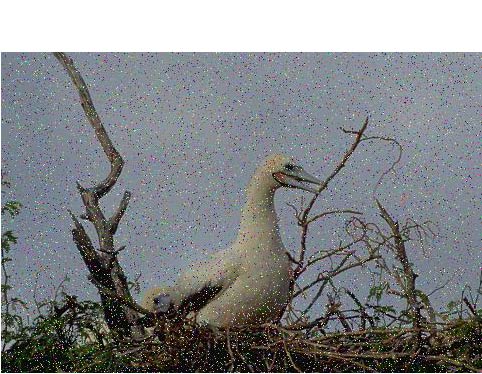}
		\caption{Observed}
	\end{subfigure}
	\begin{subfigure}[b]{0.15\textwidth}
		\includegraphics[width=\linewidth]{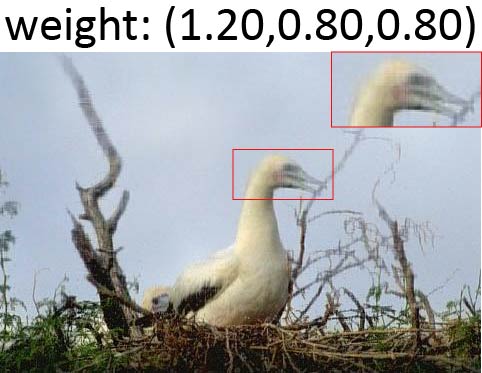}
		\caption{PSNR=26.95}
	\end{subfigure}
	\begin{subfigure}[b]{0.15\textwidth}
		\includegraphics[width=\linewidth]{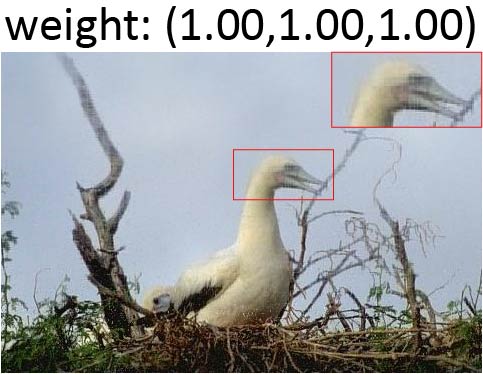}
		\caption{PSNR=28.27}
	\end{subfigure}
	\begin{subfigure}[b]{0.15\textwidth}
		\includegraphics[width=\linewidth]{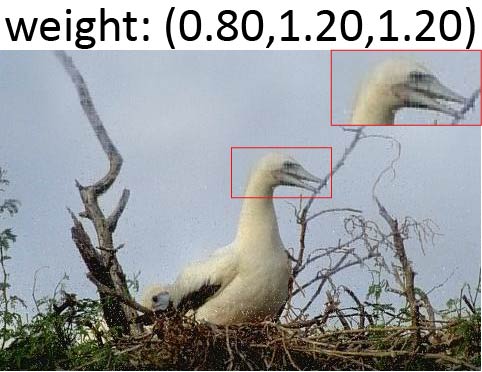}
		\caption{PSNR=29.05}
	\end{subfigure}
	\begin{subfigure}[b]{0.15\textwidth}
		\includegraphics[width=\linewidth]{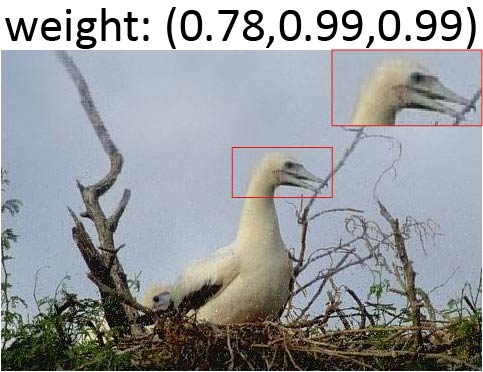}
		\caption{PSNR=29.17}
	\end{subfigure}
	\caption{Recovery performance comparison of a test image under different weight settings. (a) Original image; (b) observed image with 10\% corruption; (c)-(e) recovered images when we set $ \mbf{w}^{inter}=(1.20,0.80,0.80) $, $ \mbf{w}^{inter}=(1.0,1.0,1.0) $, $ \mbf{w}^{inter}=(0.80,1.20,1.20) $, respectively; (f) recovered image by GWTRPCA and the \textbf{adaptively} learned weight is $ \mbf{w}^{inter}=(0.78,0.99,0.99) $, which leads to the best recovery. Best viewed in ×2 sized color pdf file.}
	\label{Fig:Example_Synthetic}
\end{figure*}

The problem above has a closed-form optimal solution.
To derive the solution, we first introduce the following
lemma about the solution to the Weighted Nuclear Norm Minimization (WNNM) problem of matrix data \cite{Gu17IJCV}.
\begin{lemma}\cite{Gu17IJCV}\label{lemma:WNNM}
	Given a data matrix $\mbf{M}\in\mbb{R}^{d_1\times d_2}$
	and a weight vector
	$\mbf{w}=[w_1,\cdots,w_d]^T\in\mbb{R}^d$
	where $d=\min(d_1, d_2)$,
	let $\mbf{M}=\mbf{U}\mbf{S}\mbf{V}^{T}$ be the SVD of $\mbf{M}$.
	Consider the WNNM problem
	\begin{equation}\label{eq:WNNM}
		\mbf{prox}_{\|\cdot\|_{\mbf{w},*}}(\mbf{M}):=
		\argmin_{\mbf{L}\in\mbb{R}^{d_1\times d_2}}
		\frac{1}{2}\|\mbf{L}-\mbf{M}\|_F^2+\|\mbf{L}\|_{\mbf{w},*}
	\end{equation}
	where $\mbf{prox}_{\|\cdot\|_{\mbf{w},*}}(\mbf{M})$
	denotes the proximal operator w.r.t.
	the weighted nuclear norm $\|\cdot\|_{\mbf{w},*}$,
	$\|\mbf{L}\|_{\mbf{w},*}
	=\sum_{i=1}^dw_i\sigma_i(\mbf{L})$.
	If the weights satisfy $0\leq w_1\leq \cdots \leq w_d$ , the global solution to (\ref{eq:WNNM}) is
	$$
	\mbf{L}^*=\mbf{prox}_{\|\cdot\|_{\mbf{w},*}}(\mbf{M})
	=\mbf{U}\mbf{P}_{\mbf{w}}(\mbf{S})\mbf{V}^T,
	$$
	where $\mbf{P}_{\mbf{w}}(\mbf{S})$ is a diagonal matrix
	with the same size of $\mbf{S}$
	and $(\mbf{P}_{\mbf{w}}(\mbf{S}))_{ii}
	=(\mbf{S}_{ii}-w_i)_+$.
	Here $(x)_+=x$ if $x>0$ and $(x)_+=0$ otherwise.
\end{lemma}

For a tensor data $\mathcal{M}\in\mbb{R}^{d_1\times d_2\times d_3}$ and a weight matrix 	$\mbf{W}\in\mbb{R}^{d_3\times d}$
where $d=\min(d_1, d_2)$, let $\mathcal{M}=\mathcal{U}*\mathcal{S}*\mathcal{V}^{*}$
be the T-SVD of $\mathcal{M}$. For each $ \tau >0 $, we can define a tensor operator  $\mathcal{P}_{\tau}(\cdot)$ as
$$
(\mathcal{P}_{\tau}(\bar{\mathcal{S}}))^{(k)}
=\mbf{P}_{w_k^{inter}\mbf{w}^{intra}}\left(\bar{\mbf{S}}^{(k)}\right), 
~k\in[d_3]. 
$$
We can define the GWTNN proximal operator as follows 
\begin{equation}\label{Ope:GWTNN}
\mbf{prox}_{\|\cdot\|_{\ostar}}(\mathcal{M})
=\mathcal{U}*
ifft\left(\mathcal{P}_{\tau}(\bar{\mathcal{S}}), [], 3\right)
*\mathcal{V}^*.
\end{equation}
\begin{theorem}\label{th:GWTNN}
	For any data tensor $ \tau >0 $,
	$\mathcal{M}\in\mbb{R}^{d_1\times d_2\times d_3}$
	and a weight matrix
	$\mbf{W}\in\mbb{R}^{d_3\times d}$, if the intra-frontal slice weight vector $\mbf{w}^{intra}$ satisfies 
	$0\leq w_1^{intra}\leq \cdots w_d^{intra}$, the GWTNN proximal operator (\ref{Ope:GWTNN}) obeys
	\begin{equation}\label{eq:S2WTNN-sub}
		\mbf{prox}_{\|\cdot\|_{\ostar}}(\mathcal{M})
		=\argmin_{{\mathcal{L}}\in\mbb{R}^{d_1\times d_2\times d_3}}
		\frac{1}{2}\|{\mathcal{L}}-\mathcal{M}\|_F^2
		+\tau\|{\mathcal{L}}\|_{\ostar}.
	\end{equation}
\end{theorem}

\begin{proof}
	Please see the proof in Supplementary matrial B due to space limitation.
\end{proof}

\textbf{Step 2: } Update ${\mathcal{E}}$

\begin{equation}\label{eq:UpdateS}
	{\mathcal{E}}_{t+1}=\argmin_{{\mathcal{E}}\in\mbb{R}^{d_1\times d_2\times d_3}}
	\frac{1}{2}\left\|{\mathcal{E}}-\left(\mathcal{X}-{\mathcal{L}}_{t+1}-\mathcal{Y}_t/\mu_t\right)\right\|_F^2
	+\frac{\lambda}{\mu}\|{\mathcal{E}}\|_{1}.
\end{equation}
The problem above has a closed-form solution. Inspired by soft-thresholding operator, we have 
\begin{equation}\label{soft-thresholding}
	\mathcal{E}_{t+1}=\mathrm{T}_{\frac{\lambda}{\mu_{t}}}\left(\mathcal{H}_{t+1}\right), \mathcal{H}_{t+1}=\mathcal{X}-{\mathcal{L}}_{t+1}-\mathcal{Y}_t/\mu_t,
\end{equation}
where the $ (i,j,k)- $th element of $ \mathrm{T}_{\frac{\lambda}{\mu_{t}}}\left(\mathcal{H}_{t+1}\right) $ is $ sign\left(\left(\mathcal{H}_{t+1}\right)_{ijk}\right) \cdot \max \left(\left|\left(\mathcal{H}_{t+1}\right)_{ijk}\right|-\lambda / \mu_{t}, 0\right).$

\textbf{Step 3: } Update $\mathcal{Y}$ and $\mu$

$$
\mathcal{Y}_{t+1}=\mathcal{Y}_t+\mu_t({\mathcal{L}}_{t+1}-{\mathcal{E}}_{t+1}-\mathcal{X}), \mu_{t+1}=\min\{\rho\mu_t, \mu_{\text{max}}\}.
$$

\begin{figure*}[htbp]
	\centering
	\begin{subfigure}[b]{0.135\textwidth}
		\includegraphics[width=\linewidth]{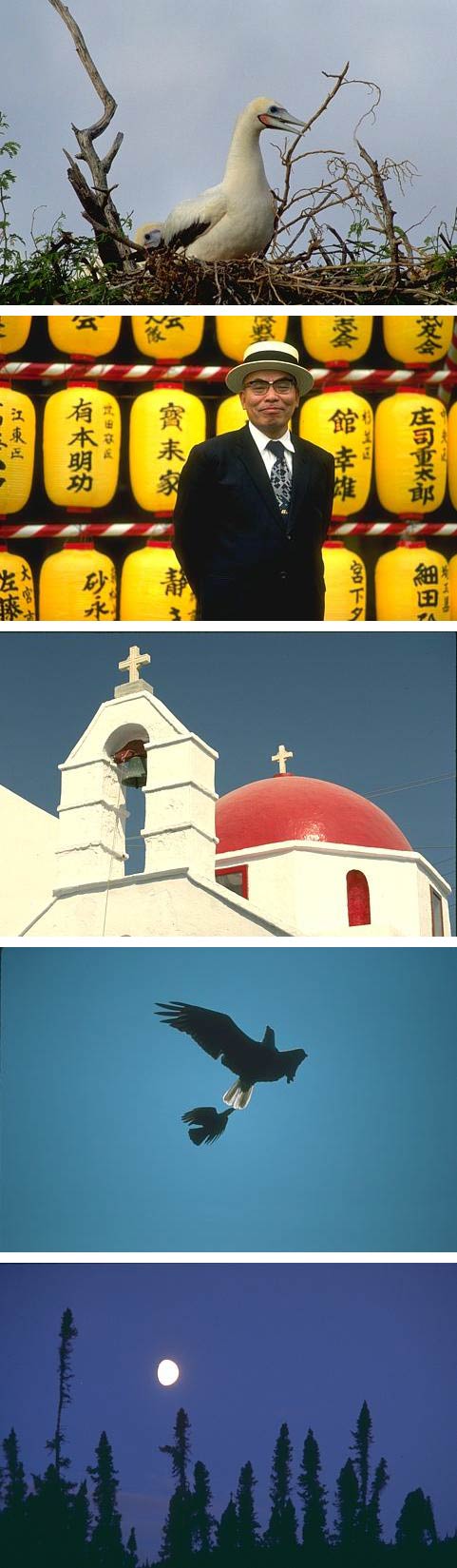}
		\caption{Original}
	\end{subfigure}
	\begin{subfigure}[b]{0.135\textwidth}
		\includegraphics[width=\linewidth]{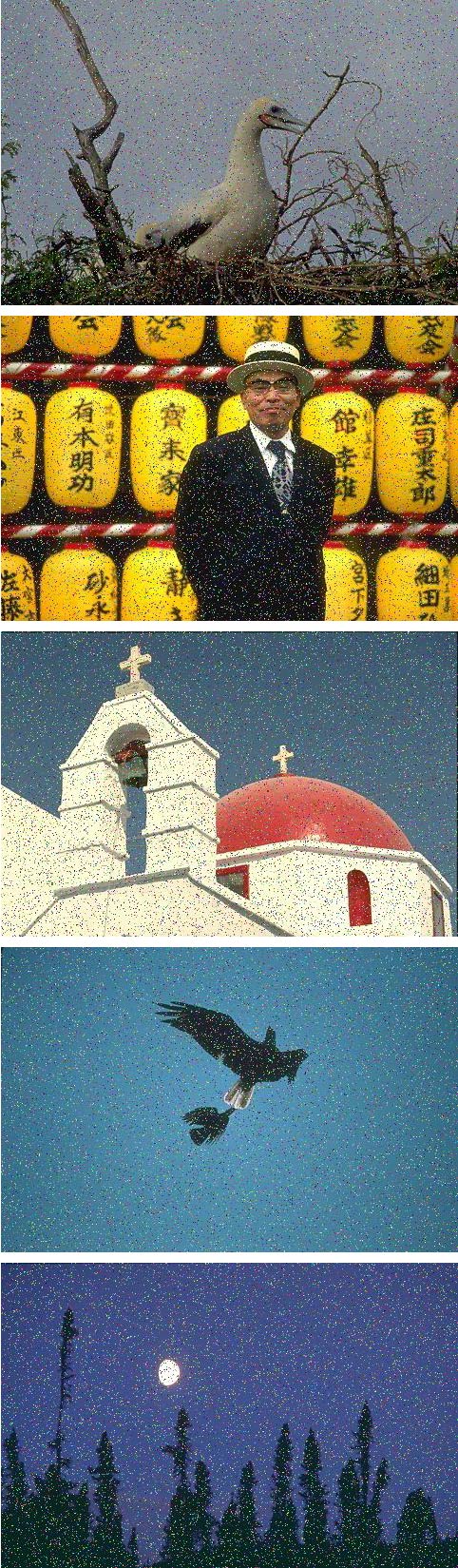}
		\caption{Observed}
	\end{subfigure}
	\begin{subfigure}[b]{0.135\textwidth}
		\includegraphics[width=\linewidth]{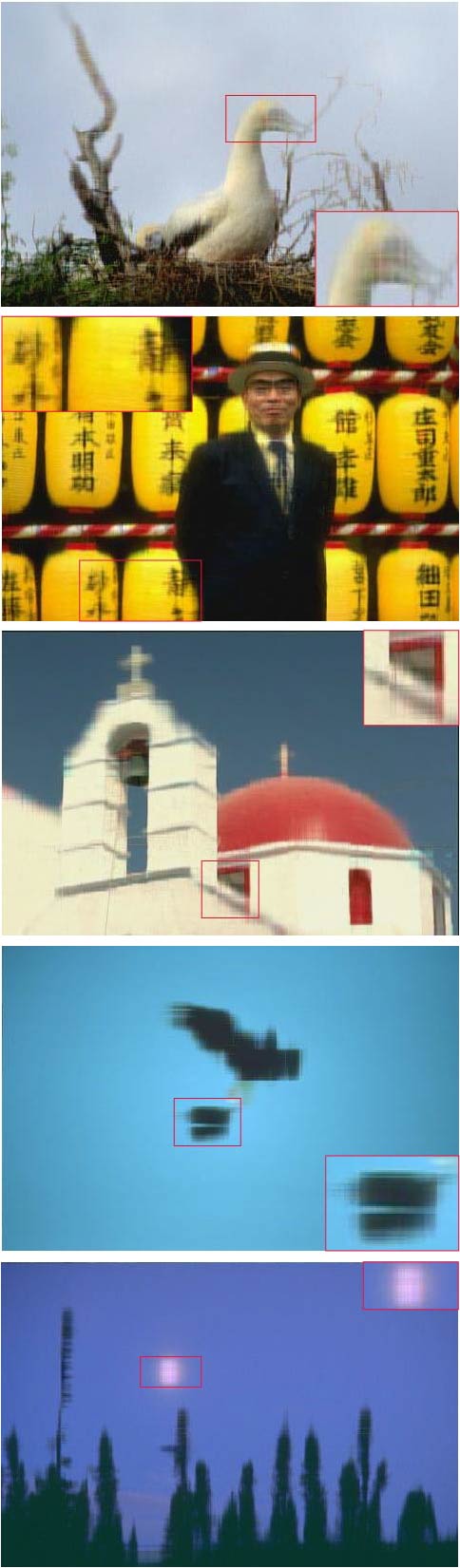}
		\caption{RPCA}
	\end{subfigure}
	\begin{subfigure}[b]{0.135\textwidth}
		\includegraphics[width=\linewidth]{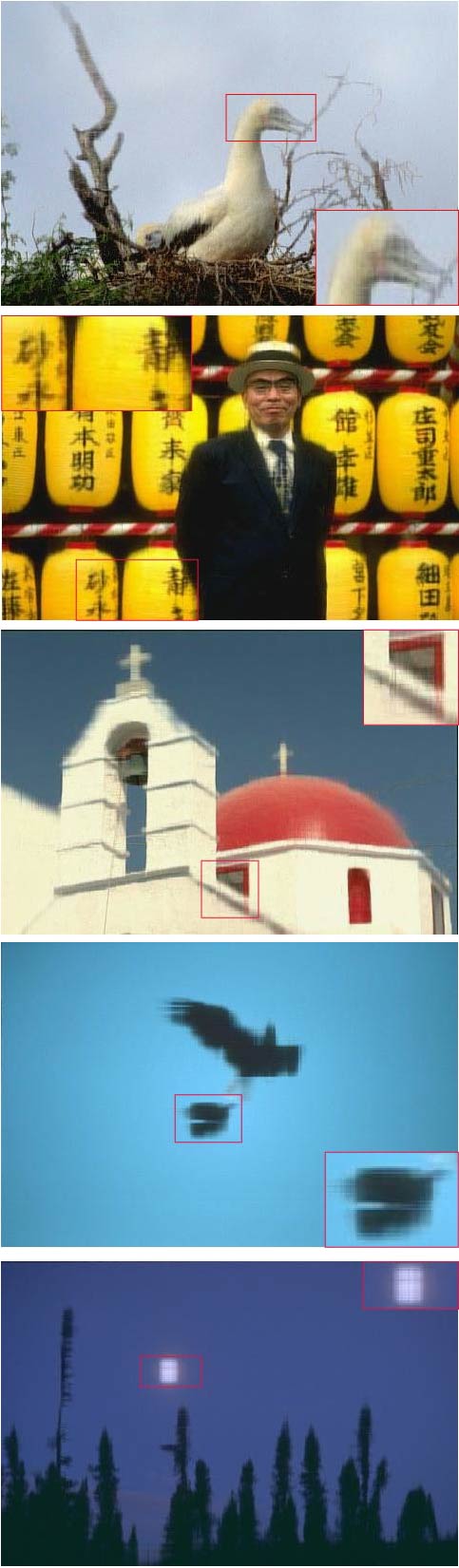}
		\caption{SNN}
	\end{subfigure}
	\begin{subfigure}[b]{0.135\textwidth}
		\includegraphics[width=\linewidth]{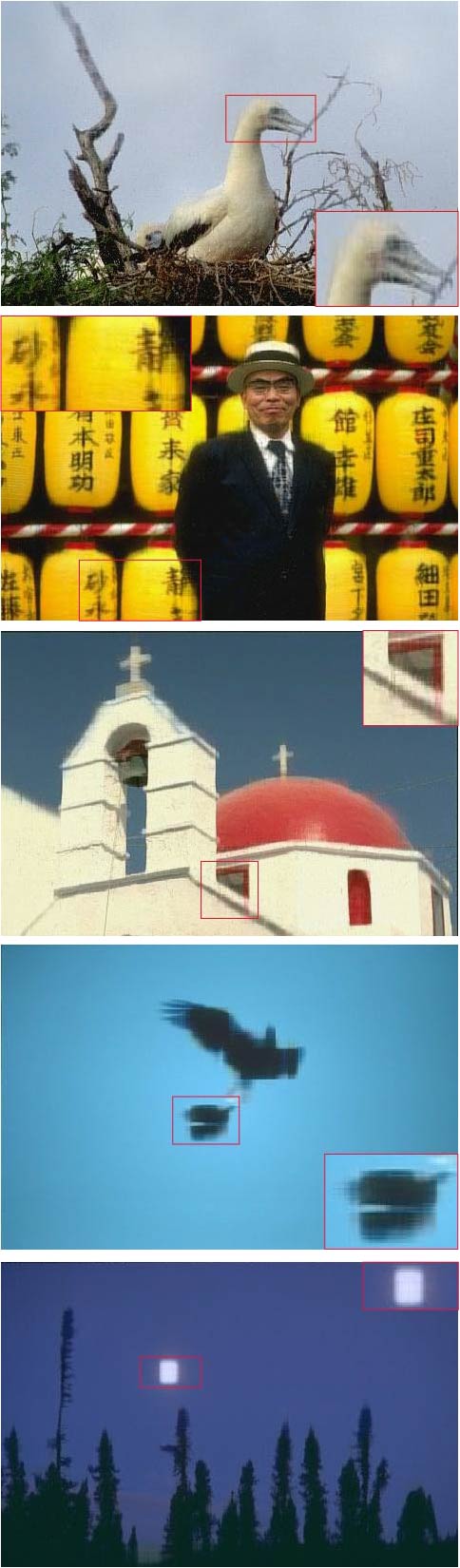}
		\caption{TRPCA}
	\end{subfigure}
	\begin{subfigure}[b]{0.135\textwidth}
		\includegraphics[width=\linewidth]{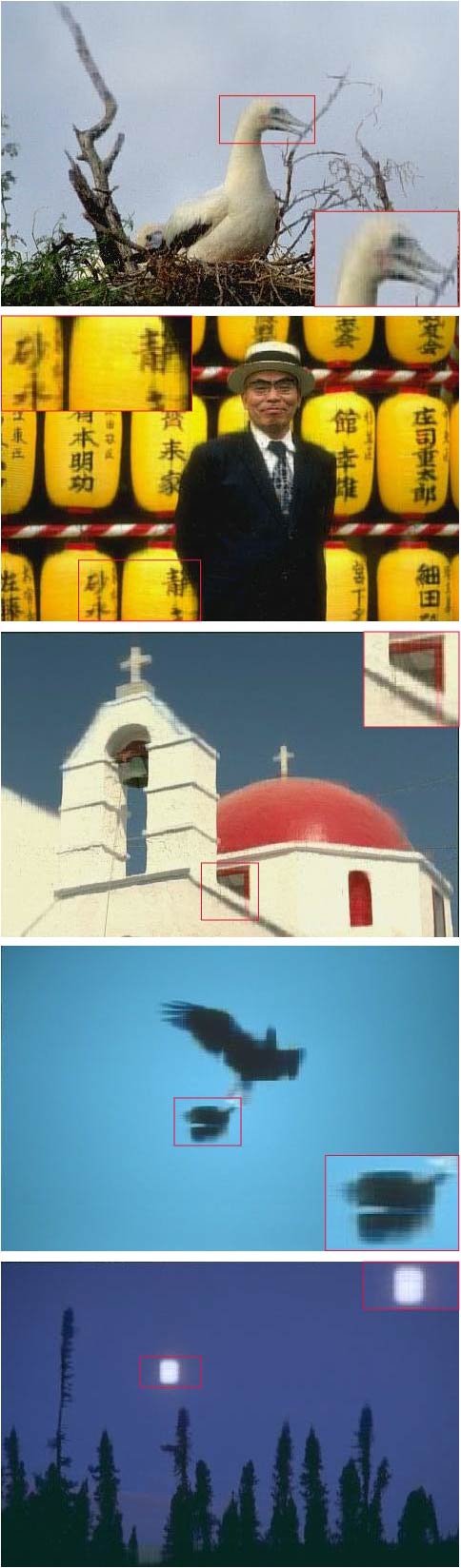}
		\caption{ETRPCA}
	\end{subfigure}
	\begin{subfigure}[b]{0.135\textwidth}
		\includegraphics[width=\linewidth]{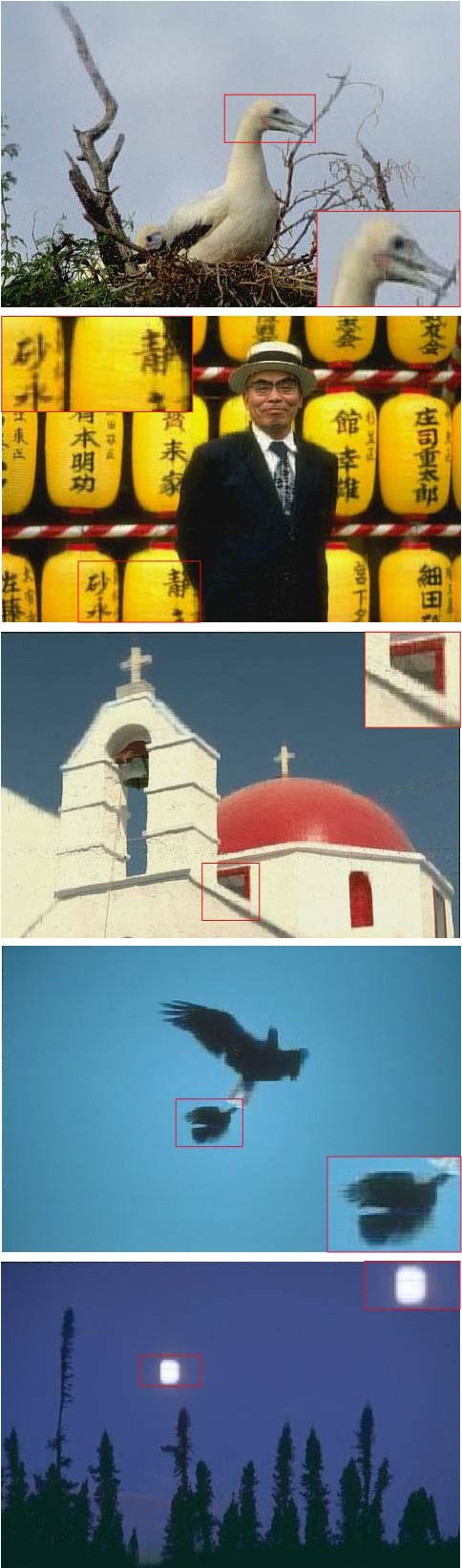}
		\caption{GWTRPCA}
	\end{subfigure} 
	\caption{Recovery performance comparison of 5 example images on the BSD dataset. (a) Original image; (b) observed image; (c)-(g) recovered images
		by RPCA, SNN, TRPCA, ETRPCA and our method GWTRPCA, respectively.The PSNR values can be seen in Tab. \ref{tab:PSNR_Example}. Best viewed in ×2 sized color pdf file.}
	\label{Fig:Example_BSD}
\end{figure*}

Secondly, we test the capability of GWTRPCA to learn adaptive weights by a modified Cauchy estimator. We use GWTRPCA to recover the same image and also set $ \mbf{w}^{intra}=\mbf{1} $. As is shown in Fig. \ref{Fig:Example_Synthetic}, the performance of our method is most outstanding. Besides, it can be seen that the weight $ \mbf{w}^{inter}$ learned by our method obeys the decreasing property mentioned in Eq. (\ref{Decrease-inter}) and (\ref{Decrease-symmetric}). Hence, GWTRPCA can achiecve better recovery by our adaptive weight learning strategy.

Now we utilize the MCE to calculate the inter-frontal slice weight as follows: 
It is worth noting that our adaptive weight learning strategy can also be employed to calculate the intra-frontal slice weight $ \mbf{w}^{intra} $. 
\begin{table}
	\begin{center}
		\normalsize
		\begin{tabular}{c|ccccc}
			\hline
			Index     & 1 & 2 & 3 & 4 & 5             \\ \hline
			RPCA & 24.51                  & 23.46                 & 24.08                   & 30.78                    & 29.95 \\ \hline
			SNN & 26.42                  & 25.20                 & 25.40                   & 26.20                    & 28.33 \\ \hline
			TRPCA & 28.28                  & 25.40                 & 27.20                   & 31.73                    & 33.73 \\ \hline
			ETRPCA & 28.49                  & 26.20                 & 27.78                   & 32.41                    & 34.74 \\ \hline
			GWTRPCA & \textbf{29.29}                  & \textbf{27.72}                 & \textbf{28.96}                   & \textbf{34.09}                    & \textbf{36.08} \\ \hline
		\end{tabular}
	\end{center}
	\caption{Comparison of PSNR values of RPCA, SNN, TRPCA, ETRPCA and GWTRPCA on 5 example images utilized in Fig. \ref{Fig:Example_BSD}. The best results are marked bold.}
	\label{tab:PSNR_Example}
\end{table} 
\section{Experiments}
In this section, we compare our method GWTRPCA with other competing methods such as RPCA \cite{RPCA}, SNN \cite{SNN}, TRPCA \cite{TRPCA} and ETRPCA \cite{ETRPCA} on the application of color image recovery and hyperspectral image denoising. We manually set the weight $ \mbf{w}^{intra} $ like ETRPCA. We devide singular values into three groups and corresponding weight is set to 0.8, 0.8 and 1.2. We utilize the Peak Signal-to-Noise Ratio (PSNR) \cite{PSNR} to evaluate recovery performance and the higher PSNR value the better recovery.

\subsection{Application to Color Image Recovery}
In this experiment, we apply GWTRPCA for color image recovery. It has been shown color images can be well approximated by low rank matrices or tensors \cite{Liu13TPAMI}. We conduct experiments on two real-world datasets: Berkeley Segmentation dataset (BSD) \cite{BSD} and Kodak image dataset (Kodak) \cite{Kodak}. For each image, we randomly set 10\% of pixels to random values in [0,255] and the positions of the corrupted pixels are unknown. All 3 channels of color images are corrupted at the same positions and the corruptions are on the sparse tubes.

We compare our method GWTRPCA with RPCA, SNN, TRPCA and ETRPCA on image recovery. For RPCA, we apply it on each channnel and receive the final recovered image by comnbining the results. We set the parameter $ \lambda=1 / \sqrt{\max \left(d_{1}, d_{2}\right)} $ as suggested in theory \cite{RPCA}. For SNN, we find that SNN does not perform well when the parameter are set to the values suggested in theory \cite{SNN}. Hence, we empirically set $ \lambda=[15,15,1.5] $ to ensure SNN can achieve great performance in most cases. For TRPCA, ETRPCA and GWTRPCA, the parameter is all set to $ \lambda=1 / \sqrt{d_{3}\max \left(d_{1}, d_{2}\right)} $.
\begin{figure}
	\includegraphics[height=5cm,width=8cm]{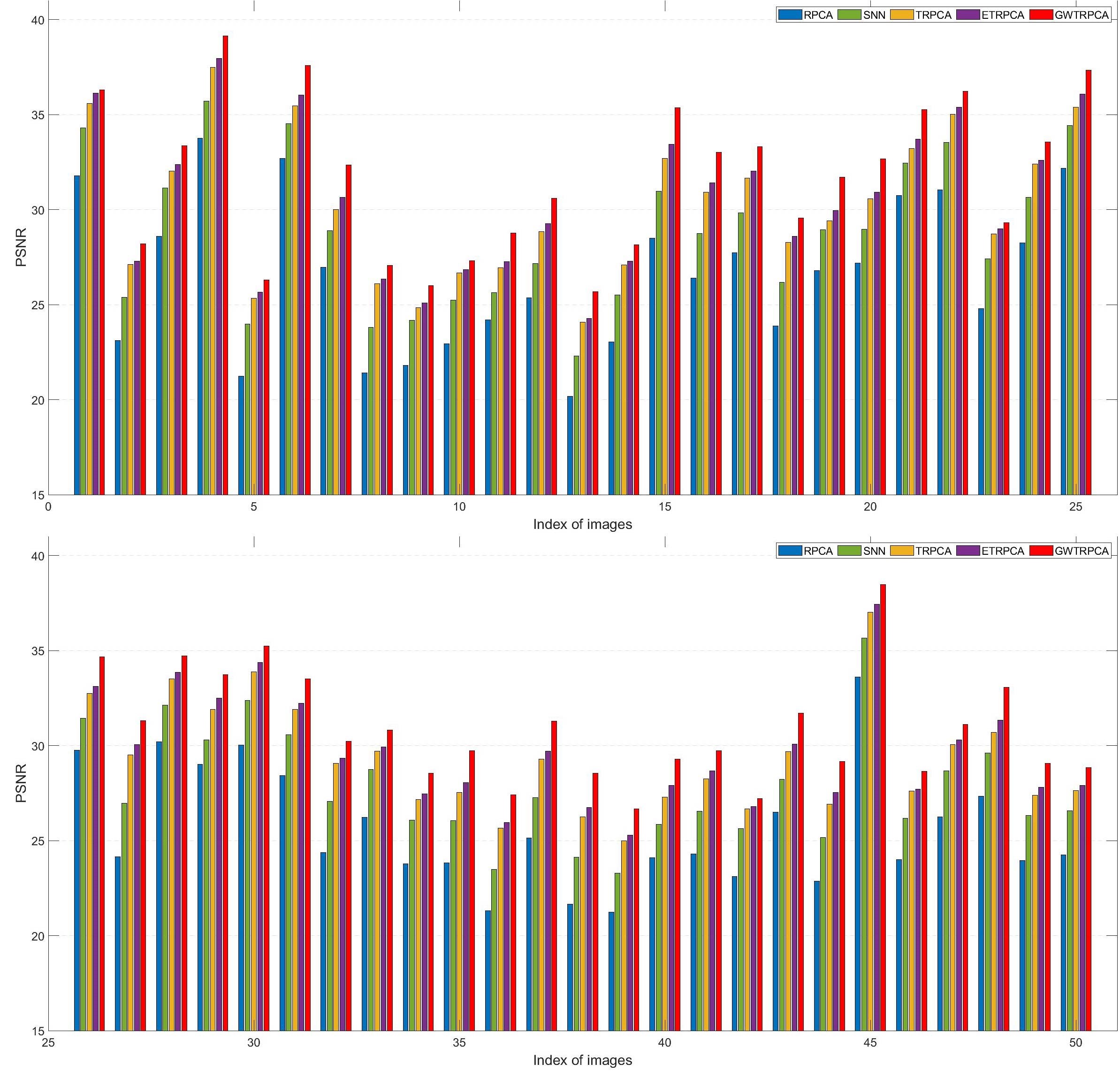}
	\caption{Comparsion of the PSNR values of RPCA, SNN, TRPCA, ETRPCA and our method GWTRPCA for image recovery on 50 images of BSD. Top row: PSNR values of the first 25 images, Bottom row: PSNR values of the remaining 25 images.}
	\label{Fig:BSD_PSNR50}
\end{figure}

\begin{table}
	\begin{center}
		\normalsize
		\begin{tabular}{c|cccc}
			\hline
			Dataset     & BSDtrain & BSDval & BSDtest & BSD            \\ \hline
			RPCA & 25.52 & 25.02 &25.10 &25.25 \\ \hline
			SNN & 27.63 & 27.20 &27.20 &27.37 \\ \hline
			TRPCA & 29.11 & 28.72 &28.63 &28.84 \\ \hline
			ETRPCA & 29.50 & 29.07 &29.03 &29.22 \\ \hline
			GWTRPCA & \textbf{30.71} & \textbf{30.12} &\textbf{30.23} &\textbf{30.40} \\ \hline
		\end{tabular}
	\end{center}
    \caption{Comparison of average PSNR values of different methods on the BSD dataset and its three subsets. BSDtrain, BSDval, BSDtest and BSD consist of 200, 100, 200 and 500 natural images, respectively. The best results are marked bold.}
	\label{tab:PSNR_ImageRecovery}
\end{table} 

Some examples of the recovered images on the BSD dataset can be seen in Fig. \ref{Fig:Example_BSD}. The PSNR values of several methods on 50 color images of BSD is shown in Fig. \ref{Fig:BSD_PSNR50}. Tab. \ref{tab:PSNR_ImageRecovery} illustrates the PSNR values of GWTRPCA and other comparison methods on the BSD dataset. More reuslts on the Kodak dataset can be seen in \emph{Supplementary Material C}. From these reuslts, we have the following conclusions. 
Firstly, tensor based methods (SNN, TRPCA, ETRPCA and GWTRPCA) are generally superior than the matrix based RPCA. The reason is that RPCA applys the image recovery on each channel independently and fails to exploit the high-order information across the channels. On the contrary, tensor based methods can employ the multi-dimensional structure of color image. Secondly, GWTRPCA can recover more edge and color information than TRPCA and ETRPCA. The reason can be that GWTRPCA considers the global information of tensor singular values and can achieve better recovery results. Besdies, ETRPCA only empirically and manually sets the weights to ensure great performance. GWTRPCA can adaptively learn the inter-frontal slice weights. 
\begin{table}
	\begin{center}
		\normalsize
		\begin{tabular}{c|cccc}
			\hline
			Noise Level     & 10\% & 20\% & 30\% & 40\%            \\ \hline
			RPCA & 30.30 & 28.58 &26.35 &23.42 \\ \hline
			SNN & 33.05 & 31.46 &29.86 &28.20 \\ \hline
			TRPCA & 45.84 & 43.82 &40.89 &35.63 \\ \hline
			ETRPCA & 45.87 & 43.90 &40.93 &37.99 \\ \hline
			GWTRPCA & \textbf{46.56} & \textbf{44.51} &\textbf{41.37} &\textbf{38.50} \\ \hline
		\end{tabular}
	\end{center}
    \caption{Comparison of average PSNR values of different methods on the Washington DC Mall HSI dataset under varying noises for HSI denoising. The best results are marked bold.}
	\label{tab:PSNR_WDC}
\end{table} 
\begin{table}
	\begin{center}
		\normalsize
		\begin{tabular}{c|cccc}
			\hline
			Noise Level     & 10\% & 20\% & 30\% & 40\%            \\ \hline
			RPCA & 27.24 & 26.30 &25.27 &23.90 \\ \hline
			SNN & 29.07 & 28.02 &26.95 &25.86 \\ \hline
			TRPCA & 41.21 & 39.84&38.40 &36.57 \\ \hline
			ETRPCA & 41.24 & 39.89 &38.44 &36.66 \\ \hline
			GWTRPCA & \textbf{48.73} & \textbf{47.26} &\textbf{45.40} &\textbf{41.78} \\ \hline
		\end{tabular}
	\end{center}
    \caption{Comparison of average PSNR values of different methods on the Pavia University HSI dataset under varying noises for HSI denoising. The best results are marked bold.}
	\label{tab:PSNR_PaviaU}
\end{table} 
\subsection{Application to Hyperspectral Image Denoising}
In this part, we verify the effectiveness of GWTRPCA on hyperspectral image (HSI) denoising. We conduct experiments on the Washington DC Mall \footnote{https://engineering.purdue.edu/$\sim$biehl/MultiSpec/hyperspectral\\.html} 
and Pavia University HSI datasets\footnote{http://www.ehu.es/ccwintco/index.php/Hyperspectral\_Remote \\ \_Sensing\_Scenes}, which can be treated as a tensor of size $ 256\times256\times150 $ and $ 610\times340\times103 $, respectively.

We add random salt-pepper noise to each hyperspectral image and the noise level is set from 10\% to 40\%. The average PSNR values of HSI denoising of GWTRPCA and other competitors are reported in Tab. \ref{tab:PSNR_WDC} and Tab. \ref{tab:PSNR_PaviaU}. It's easy to see that GWTRPCA outperforms other methods. The denoising results can be worse with the increase of noise level. Despite this, our method earns more outstanding results even when the noise level is set to 40\%.

\section{Conclusion}
In this paper, we propose the Global Weighted Tensor Robust Principal Component Analysis (GWTRPCA) method based on a new defined Global Weighted Tensor Nuclear Norm (GWTNN). We simultaneously consider the importance of intra-frontal slice and
inter-frontal slice singular values. With the aid of global information of singular values, GWTRPCA can approximate the tubal rank function more exactly. Furthermore, we utilize a MCE to adaptively learn inter-frontal slice weights. An effective algorithm is devised to solve the GWTRPCA optimization problem based on the ADMM framework. Finally, the experiments on real-world databases validate the effectiveness of the proposed method.
\bibliography{Refs}

\end{document}